\documentclass[]{interact}

\usepackage{epstopdf}
\usepackage{subfigure}
\usepackage{natbib}
\usepackage{times}
\usepackage{xcolor}
\usepackage{enumitem}
\usepackage{hyperref} 
\usepackage{url} 
\usepackage{amsthm}
\bibpunct[, ]{(}{)}{;}{a}{}{,}
\newcommand{\z}{{\bf z}}
\newcommand{\y}{{\bf y}}
\newcommand{\X}{\mathcal{X}}
\newcommand{\Hs}{\mathcal{H}}
\newcommand{\Y}{\mathcal{Y}}
\newcommand{\M}{\mathcal{M}_1^+}
\newcommand{\B}{\mathcal{B}}
\newcommand*\diff{\mathop{}\!\mathrm{d}}
\newcommand{\norm}[1]{\left\lVert#1\right\rVert}
\newcommand{\innerpro}[1]{\left\langle#1\right\rangle}
\DeclareMathOperator*{\argmin}{arg\,min}
\renewcommand{\x}{{\bf x}}

\theoremstyle{plain}
\newtheorem{theorem}{Theorem}
\newtheorem{lemma}[theorem]{Lemma}
\newtheorem{corollary}[theorem]{Corollary}

\theoremstyle{definition}

\theoremstyle{remark}
\newtheorem{remark}{Remark}

\begin{document}

\articletype{under review}

\title{Domain Generalization by Functional Regression}

\author{
\name{Markus Holzleitner\textsuperscript{a}, Sergei V.~Pereverzyev\textsuperscript{b} and Werner Zellinger\textsuperscript{b}\thanks{CONTACT Werner Zellinger. Email: werner.zellinger@oeaw.ac.at}}
\affil{\textsuperscript{a}Institute for Machine Learning, Johannes Kepler University Linz, Austria; \textsuperscript{b}Johann Radon Institute for Computational and Applied Mathematics, Austrian Academy of Sciences}
}

\maketitle

\begin{abstract}
The problem of domain generalization is to learn, given data from different source distributions, a model that can be expected to generalize well on new target distributions which are only seen through unlabeled samples.
In this paper, we study domain generalization as a problem of functional regression.
Our concept leads to a new algorithm for learning a linear operator from marginal distributions of inputs to the corresponding conditional distributions of outputs given inputs.
Our algorithm allows a source distribution-dependent construction of reproducing kernel Hilbert spaces for prediction, and, satisfies finite sample error bounds for the idealized risk.
Numerical implementations and source code are available\footnote{\label{source}Source code: \url{https://github.com/wzell/FuncRegr4DomGen}}.
\end{abstract}


\section{Introduction}

Most problems in learning theory assume identically distributed data.
In contrast, in domain generalization~\cite{blanchard2011generalizing,muandet2013domain,blanchard2021domain}, the learner observes $N\in\mathbb{N}$ data samples\footnote{A sample $(x_i,y_i)_{i=1}^n,n\in\mathbb{N}$ from a probability measure $P$ consists of realizations $(x_i,y_i):=Z(\omega_i)\in\X\times\Y, i\in\{1,\ldots,m\}$ from a random variable $Z$ with measure $P$ on $\X\times\Y$ at independent events $\{\omega_1\},\ldots,\{\omega_m\}$.} $$(x_i^{(1)},y_i^{(1)})_{i=1}^{n},\ldots,(x_i^{(N)},y_i^{(N)})_{i=1}^{n}\in \left((\X\times \Y)^n\right)^N$$
which follow $N$ different distributions $P^{(1)},\ldots,P^{(N)}$.
These \textit{source} distributions model different real-world \textit{domains}, e.g., different medical patients.
The goal of domain generalization is to find a model $g:(x_i^{T})_{i=1}^{n}\mapsto (\widehat{f}:\X\to\Y)$ that is able to derive a predictor $\widehat{f}:\X\to\Y$ from only \textit{unlabeled} data $(x_i^{T})_{i=1}^{n}\in\X^n$ following a new \textit{target} distribution $P^T$.
The performance of $g$ is quantified in expectation over a random draw of the target distribution $P^T$, i.e., by $\mathcal{E}^\infty(g):=\int \int (\widehat{f}(x)-y)^2\diff P^T(x,y)\diff E(P^T)$ for some meta-distribution (or \textit{environment}~\cite{baxter1998theoretical}) $E$ from which $P^T$ is drawn.

This work is concerned with the worst-case sample complexity of domain generalization.
More precisely, we study the question of \textit{whether and how the model $g$ above can be computed from the $N$ given samples, such that it satisfies small bounds on $\mathcal{E}^\infty(g)-\inf_{h}\mathcal{E}^\infty(h)$?}

Although many finite sample results are available in the similar settings of meta-learning (cf.~\cite{maurer2005algorithmic}) and multi-task learning (cf.~\cite{evgeniou2005learning}), relatively less is known for domain generalization.
Indeed, the inaccessibility of target labels requires new techniques for domain generalization.
For example, to learn from unlabeled data of a target distribution $P^T$, it is required to relate the marginal distribution $P_\X^T(x)$ (of inputs $x\in\X$) to the conditional distribution $P^T(y|x)$ (of outputs $y\in\Y$ w.r.t.~inputs), where $P^T(x,y)=P_\X^T(x) P^T(y|x)$.
This is implicitly done in the seminal \textit{marginal transfer} approach of~\cite{blanchard2011generalizing,blanchard2021domain}, where a \textit{universal consistent} algorithm is proposed, i.e., an algorithm computing $g$ such that $\mathcal{E}^\infty(g)\to \inf_{h}\mathcal{E}^\infty(h)$ almost surely for $n,N\to\infty$.
However, to the best of our knowledge, specifying the rate of this convergence is still an open research problem.

The main conceptual contribution of this work is to recast domain generalization as a problem of functional regression, which allows for analytical results from that field.
More precisely, we propose a new algorithm for \textit{explicitly} learning (from the input samples) an operator, which maps (kernel mean embeddings of) input marginal distributions $P_\X(x)$ to approximations of the regression functions (Bayes predictors) $f_{P}:=\argmin_{f:\X\to\Y}\int_{\X\times\Y} (f(x)-y)^2\diff P(x,y)$ of the corresponding conditional distributions $P(y|x)$.
We, here, in a first step, focus on learning a linear operator with slope functions residing in an RKHS, which allows us to apply analytical arguments from~\cite{mollenhauer2022learning,jin2022minimax,tong2022non} resulting in explicit finite sample bounds.
However, our new concept opens new directions for domain generalization by linear and non-linear operator learning.

Another, particularly practical, advantage of our method is the possibility to choose different reproducing kernel Hilbert spaces (RKHSs) for regression on different source distributions $(P^{(i)})_{i=1}^N$.
This allows expert-choices and automation, e.g., by choosing between well-known kernels by cross-validation.
We provide a numerical example which illustrates this advantage and gives a simple implementation of the proposed algorithm.
Our contributions can be summarized as follows:\begin{itemize}
    \item We provide a new concept for approaching domain generalization by functional regression.
    \item We propose a new algorithm which allows a domain-specific data-based construction of predictors, e.g., different learned RKHSs for different domains. As one consequence, new target predictors are not needed to be well approximable by pre-defined (e.g.~Gaussian) RKHSs.
    \item We provide (to the best of our knowledge first) finite sample bounds for $\mathcal{E}^\infty(g)-\inf_{h}\mathcal{E}^\infty(h)$ in the domain generalization setting of~\cite{blanchard2011generalizing}.
    \item We provide a numerical implementation showing the advantage of our algorithm.
\end{itemize}

\section{Background on Domain Generalization}

\subsection{Domain Generalization}

Let $\X\subset\mathbb{R}^{d}$ be a compact \textit{input} space (with Lebesgue measure one for simplicity) and $\Y\subset\mathbb{R}$ be a compact \textit{output} space.
The problem of domain generalization~\cite{blanchard2011generalizing,muandet2013domain,blanchard2021domain} extends the problem of supervised learning by relaxing the assumption of one unique underlying data distribution.
In particular, in domain generalization, we have given a vector
\begin{align}
    \label{eq:source_samples}
    (\z^{(i)})_{i=1}^N:=
    (\x^{(i)},\y^{(i)})_{i=1}^N:=\left((x_j^{(i)},y_j^{(i)})_{j=1}^{n_i}\right)_{i=1}^N\in \left(\bigcup_{n=1}^\infty\left(\X\times\Y\right)^n\right)^N
\end{align}
of \textit{source} samples, drawn independently at random according to $N\in\mathbb{N}$ respective probability measures $P^{(1)},\ldots,P^{(N)}$ from the set $\M(\X\times\Y)$ of probability measures on $\X\times\Y$.
For convenience, we represent a sample $\z^{(i)}$ by its associated empirical probability measure $\widehat{P}^{(i)}:=\frac{1}{n_i}\sum_{j=1}^{n_i} \delta_{(x_j^{(i)},y_j^{(i)})}\in\M(\X\times\Y)$, where $\delta_z$ is the Dirac delta function on $z\in\X\times\Y$.
The goal in domain generalization is to construct an algorithm
\begin{align}
\label{eq:domain_generalization_algorithm}
    A:\left(\M(\X\times\Y)\right)^N &\to
    \left\{ g: \M(\X) \to \{f:\X\to\Y\}\right\}
\end{align}
which maps the $N$ source samples $\widehat{P}^{(1)},\ldots,\widehat{P}^{(N)}$ to a function
\begin{align}
\label{eq:domain_generalization_function_g}
g:\M(\X) \to \{f:\X\to\Y\}
\end{align}
that needs only an \textit{unlabeled} target sample $\x^T=(x_j^T)_{j=1}^{n_T}$, represented by $\widehat{P}_\X^{T}\in\M(\X)$ and
drawn independently at random according to some (marginal) probability measure $P_\X^{T}\in\M(\X)$, to infer a predictor $g(\widehat{P}_\X^T):=f:\X\to\Y$ that performs well on new data  $(x,y)$ drawn (independently from $\x^T$) according to $P^T$~\citep{blanchard2011generalizing,blanchard2021domain}.

In the \textit{two-stage generative model of domain generalization}~\citep[Assumption~2]{blanchard2021domain}, the probability measures $P^{(1)},\ldots,P^{(N)},P^T$ are drawn independently at random according to a \textit{meta} probability measure $E$ on $\M(\X\times\Y)$
\footnote{If we equip $\M(\X\times\Y)$ with $\tau_w(\X\times\Y)$, the weakest topology on $\M(\X\times\Y)$ such that the mapping $L_h:(\M(\X\times\Y),\tau_w(\X\times\Y))\to\mathbb{R}$ with $L_h(P)=\int_{\X\times\Y} h(x,y)\diff P(x,y)$ is continuous for all bounded and continuous functions $h:\X\times\Y\to\mathbb{R}$ and denote by $\B(\tau_w(\X\times\Y))$ the associated Borel $\sigma$- algebra, then $(\M(\X\times\Y),\B(\tau_w(\X\times\Y)))$ becomes a itself measurable space, cf.~\cite{maurer2005algorithmic,szabo2016learning}.}
, and, the quality of the prediction of the model $g:=A(\widehat{P}^{(1)},\ldots,\widehat{P}^{(N)})$ is measured by the \textit{idealized risk}
\begin{align}
    \label{eq:idealized_risk}
    \mathcal{E}^{\infty}(g)=\int_{\M(\X)} \int_{\X\times\Y} \left(g(P_\X)(x)-y\right)^2\diff P(x,y) \diff E(P).
\end{align}
The choice of the error $\mathcal{E}^{\infty}(g)$ models the goal of domain generalization to find (in expectation over the choice of $P^T$) a model $f=g(P_\X^T)$ with a low expected target risk $\int_{\X\times\Y} (f(x)-y)^2\diff P^T$.

\subsection{Marginal Transfer Learning}
\label{subsec:marginal_transfer_learning}

In the seminal works~\cite{blanchard2011generalizing,muandet2013domain,blanchard2021domain}, the predictor $g(\widehat{P}_\X^T):\X\to\Y$
is defined by $g(\widehat{P}_\X^T)(x):=f_{\z^{(1)},\ldots,\z^{(N)}}(\widehat{P}_\X^T,x)$ for a
\begin{align}
    \label{eq:domain_generalization_function_g_augmented}
    f_{\z^{(1)},\ldots,\z^{(N)}}:\M(\X)\times \X\to\Y
\end{align}
\sloppy that is computed from the (by the input marginals $\widehat{P}_\X^{(i)}$ "augmented") data samples $\left((\widehat{P}_\X^{(1)}, x_j^{(1)}),y_j^{(1)}\right)_{j=1}^{n_1},\ldots, \left((\widehat{P}_\X^{(N)}, x_j^{(N)}),y_j^{(N)}\right)_{j=1}^{n_N}$.
This approach is referred to as \textit{marginal transfer learning}.
More precisely,~\cite{blanchard2021domain} follow~\cite{evgeniou2005learning} and use an RKHS $\mathcal{H}_{\overline{k}}$ generated by a kernel $\overline{k}$ on $\M(\X)\times\X$ defined by
$\overline{k}((P^{(1)},x_1), (P^{(2)},x_2)):=k_{\M(\X)}(P^{(1)},P^{(2)})\cdot k_\X(x_1,x_2)$,
where $k_{\M(\X)}$ is a kernel on $\M(\X)$ and $k_\X$ is a kernel on $\X$.
The model $f_{\z^{(1)},\ldots,\z^{(N)}}=f_{\z^{(1)},\ldots,\z^{(N)}}^\lambda$ in Eq.~\eqref{eq:domain_generalization_function_g_augmented} is computed by penalized risk estimation
\begin{align}
    \label{eq:dom_gen_kernel_least_squares}
    f_{\z^{(1)},\ldots,\z^{(N)}}^\lambda:= \argmin_{f\in\mathcal{H}_{\overline{k}}}
    \frac{1}{N}\sum_{i=1}^N \frac{1}{n_i}\sum_{j=1}^{n_i}
    \left(f(\widehat{P}_\X^{(i)},x_j^{(i)})-y_j^{(i)}\right)^2+\lambda \norm{f}_{\mathcal{H}_{\overline{k}}}^2.
\end{align}
Using $\mathcal{H}_{\overline{k}}$,~\cite{blanchard2021domain} prove convergence in idealized risk of the estimator in Eq.~\eqref{eq:dom_gen_kernel_least_squares}.
More precisely, they prove in Theorem~15 and Corollary~16, for $g_{\z^{(1)},\ldots,\z^{(N)}}^\lambda(P)(x):=f_{\z^{(1)},\ldots,\z^{(N)}}^\lambda(P,x)$, the convergence
\begin{align}
\label{eq:domain_generalization_consistency}
\mathcal{E}^\infty(g_{\z^{(1)},\ldots,\z^{(N)}}^{\lambda})\to \inf_{g:\M(\X)\to\left\{f:\X\to\mathbb{R}\right\}}\mathcal{E}^{\infty}(g)
\end{align}
in probability for $N\to\infty$, when the sample sizes $n_1,...,n_N$ are randomly drawn,  under rather general conditions on $\overline{k},\mathcal{X}$ and under a suitable choice $\lambda=\lambda(N)$.
This consistency is interesting because it allows us to hope for a small target error (in expectation w.r.t.~the random draw of $P^T$) for a \textit{sufficiently large number} $N$ of source samples $\z^{(1)},\ldots,\z^{(N)}$ of sufficiently large sample sizes $n_1,\ldots,n_N$, respectively.

\section{Problem}
\label{sec:problem}

Two issues appear: The first issue, concerns the final predictor $g(\widehat{P}_\X^T)(.)=f_{\z^{(1)},\ldots,\z^{(N)}}^{\lambda(N)}(\widehat{P}_\X^T,\cdot)$ computed as defined in Eq.~\eqref{eq:dom_gen_kernel_least_squares}, which resides in the pre-defined space $\mathcal{H}_{k_\X}$ (e.g., in~\cite{blanchard2021domain} defined by a Gaussian kernel $k_\X$ with fixed bandwidth), see Remark \ref{rem:generality}.
The space $\mathcal{H}_{k_\X}$ therefore needs to be a good choice for all domains, which might be hard to find in practice.
For example the regression functions $f_{P^{(i)}},f_{P^{(j)}}$ of two domains $i,j\in\{1,\ldots,N\}$ can reside in two different RKHSs $\Hs_{k^{(i)}},\Hs_{k^{(j)}}$ with two well-known (or well learnable) kernels $k^{(i)},k^{(j)}$, but a priori guesses for aggregations of the two kernels might lead to unstable behavior of the regression in Eq.~\eqref{eq:dom_gen_kernel_least_squares}.
The second issue concerns the rate of the convergence in Eq.~\eqref{eq:domain_generalization_consistency}, which is unknown.
To the best of our knowledge, no domain generalization algorithm is known with quantified convergence rate of Eq.~\eqref{eq:domain_generalization_consistency}. 

This work presents a domain generalization algorithm $A$ as in Eq.~\eqref{eq:domain_generalization_algorithm} (i.e., mapping the source samples $\z^{(1)},\ldots,\z^{(N)}$ to a function $g_{\z^{(1)},\ldots,\z^{(N)}}:\widehat{P}_\X^T\mapsto (f:x\mapsto y)$)
that allows one to choose different RKHSs $\mathcal{H}_{k^{(i)}}$ with kernels $k^{(i)}$ for each domain $i\in\{1,\ldots,N\}$, and, which has a quantified rate of the convergence in Eq.~\eqref{eq:domain_generalization_consistency} (i.e., of $\mathcal{E}^\infty(g_{\z^{(1)},\ldots,\z^{(N)}})\to \inf_{g:\widehat{P}_\X^T\mapsto (f:x\mapsto y)} \mathcal{E}^\infty(g)$ for increasing number of samples $\z^{(1)},\ldots,\z^{(N)}$ with increasing sizes.
    
\section{Summary of Results}


\subsection{Linear Operator Ansatz} \label{subsec:linear_ansatz}

Our approach follows the general Ansatz that there is a linear operator $G:L^2(\X)\to L^2(\X)$ mapping, for every $P\in\M(\X\times \Y)$ drawn from $E$, the \textit{kernel mean embedding}
\footnote{
It holds that $m_{P_\X}\in \Hs_k\subseteq L^2(\X)$, the space of square-integrable functions on $\X$. 
The mapping $m:P_\X\mapsto m_{P_\X}$ is well-defined if the kernel $k$ is bounded
and it is injective if $k$ is universal~\cite{gretton2006kernel,sriperumbudur2010hilbert}.}
\begin{align}
\label{eq:kernel_mean_embedding}
    m_{P_\X}(\cdot):=\int_\X k(\cdot,x')\diff P_\X(x')
\end{align}
to the domain-specific \textit{regression function}
\footnote{The regression function $f_P$ is well-defined since $\X$ and $\Y$ are Polish spaces (as compact subsets of $\mathbb{R}^{d_1},\mathbb{R}^{d_2}$) and, therefore, every $P\in\M(\X\times\Y)$ can be factorized $P(x,y)=P(y|x) P_\X(x)$ in a conditional probability measure $P(y|x)$ and a marginal (w.r.t.~$\X$) probability measure $P_\X(x)$, see~\cite[Theorem~10.2.1]{dudley2018real}.}
\begin{align}
\label{eq:regression_function_task_specific}
f_P(\cdot):=\int_\Y y\diff P(y|\cdot),
\end{align}
such that
\begin{align}
\label{eq:linear_operator_ansatz}
    f_{P}(\cdot)=G \cdot m_{P_\X}(\cdot) + \varepsilon(\cdot),
\end{align}
where $\varepsilon$ is some functional \textit{noise} that is drawn independent from $P_\X$ according to some probability measure $\mathcal{N}\in \M(L^2(\X))$, has zero mean $\int_{L^2(\X)}\varepsilon \diff \mathcal{N}(\varepsilon)\equiv 0$ and finite variance $\sigma^2:=\int_{L^2(\X)}\norm{\varepsilon}_{L^2(\X)}^2 \diff \mathcal{N}(\varepsilon)<\infty$.

\begin{remark}
    The main conceptual contribution of this paper is the Ansatz above, which recasts domain generalization as a problem of functional regression.
    More precisely, our Ansatz allows domain generalization by learning the operator $G$, which maps
    (functional) mean embeddings to domain-specific regression functions.
    This enables to estimate the regression functions in each domain by different kernels, and, to apply explicit finite-sample bounds from the field of functional regression, e.g.~\cite{mollenhauer2022learning,jin2022minimax,tong2022non}.
    The independence of the noise is for simplicity and can be removed at the price of slightly more involved proofs, see~\cite[Eq.~(1)]{jin2022minimax}.
\end{remark}
We further assume that the integral operator $G$ is of the form
\begin{align}
    \label{eq:integral_form_of_G}
    G \cdot m_{P_\X}(\cdot) := a_0(\cdot)+ \int_\X m_{P_\X}(x) \beta(\cdot, x)\diff x
\end{align}
with \textit{intercept} $a_0:\X\to\Y$ and \textit{slope} $\beta:\X\times\X\to \mathbb{R}$ that need to be learned from the data. 

\subsection{New Algorithm} \label{subsec:algorithm}

For simplicity, in the following, we assume equal sample sizes $n:=n_1=\ldots=n_N=n_T$.
Following our Ansatz in Eq.~\eqref{eq:linear_operator_ansatz} and Eq.~\eqref{eq:integral_form_of_G}, we propose the following two-step procedure:
\begin{enumerate}
    \item Regularized estimation of $f_{P^{(i)}}$ for every source sample $\z^{(i)}$ of domain $i\in\{1,\ldots,N\}$
    \begin{align}
        \label{eq:domain_specific_ridge_regression}
        f_{\z^{(i)}}^{\lambda_i}:=\argmin_{f\in\Hs_{k^{(i)}}} \sum_{j=1}^{n} (f(x_j^{(i)})-y_j^{(i)})^2 + \lambda_i \norm{f}^2_{\Hs_{k^{(i)}}}.
    \end{align}
    \item Regularized estimation of slope $\beta$ on
    functional data $(m_{\x^{(i)}},f_{\z^{(i)}}^{\lambda_i})_{i=1}^N$ with
    $m_{\x^{(i)}}:=m_{\widehat{P}^{(i)}_\X}$
\end{enumerate}
\vspace{-12pt}
\begin{align}
        \label{eq:functional_regression_of_slope}
        \beta_{\z^{(1)},\ldots,\z^{(N)}}^{\lambda_1,\ldots,\lambda_N,\lambda}:=
        \argmin_{\beta(x,\cdot)\in \Hs_k}
        \frac{1}{N}\sum_{i=1}^N
        \norm{f_{\z^{(i)}}^{\lambda_i}-\int \beta(\cdot, x')m_{\x^{(i)}}(x')\diff x' }_{L^2(\X)}^2+\lambda \int_\X \norm{\beta(x,\cdot)}_{\Hs_k}^2 \diff x.
    \end{align}
 \noindent
We define the final model $g_{\z^{(1)},\ldots,\z^{(N)}}^{\lambda_1,\ldots,\lambda_N,\lambda}:\M(\X)\to\{f:\X\to\Y\}$ as required in Eq.~\eqref{eq:domain_generalization_function_g} by $g_{\z^{(1)},\ldots,\z^{(N)}}^{\lambda_1,\ldots,\lambda_N,\lambda}({P}_\X)(x):=G_{\z^{(1)},\ldots,\z^{(N)}}^{\lambda_1,\ldots,\lambda_N,\lambda} m_{P_\X}(x)$, where $G_{\z^{(1)},\ldots,\z^{(N)}}^{\lambda_1,\ldots,\lambda_N,\lambda}$ is the integral operator defined in Eq.~\eqref{eq:integral_form_of_G} with the slope $\beta_{\z^{(1)},\ldots,\z^{(N)}}^{\lambda_1,\ldots,\lambda_N,\lambda}$.

\begin{remark} \label{rem:generality}
    Note that the final predictor $g_{\z^{(1)},\ldots,\z^{(N)}}^{\lambda_1,\ldots,\lambda_N,\lambda}(\widehat{P}^T_\X)(.)$ defined above is not enforced to reside in an RKHS $\mathcal{H}_{k^{(i)}}$ defined by one of the pre-defined kernels $k^{(1)},\ldots,k^{(N)}$ but is allowed to take more general forms.
    In this way, our algorithm can be interpreted as an extension of the marginal transfer approach, which computes a predictor $g(\widehat{P}_\X^T)(\cdot)=f_{\z^{(1)},\ldots,\z^{(N)}}^{\lambda(N)}(\widehat{P}_\X^T,\cdot)$ by Eq.~\eqref{eq:dom_gen_kernel_least_squares}.
    According to the representer theorem, see e.g.~\cite[Theorem~1]{scholkopf2001generalized}, this predictor $f_{\z^{(1)},\ldots,\z^{(N)}}^{\lambda(N)}(\widehat{P}_\X^T,\cdot)$ admits the representation
    \begin{align*}
        \sum_{i=1}^N \sum_{j=1}^{n_i} \alpha_{i,j} \cdot\overline{k}\left((\widehat{P}_\X^{(i)},x_j^{(i)}),(\widehat{P}_\X^T,\cdot)\right)
        = \sum_{i=1}^N \sum_{j=1}^{n_i} \underbrace{\alpha_{i,j}\cdot k_{\M(\X)}(\widehat{P}_\X^{(i)}, \widehat{P}_\X^T)}_{\text{$=:\widetilde{\alpha}_{i,j}\in\mathbb{R}$}} \cdot k_\X(x_j^{(i)},\cdot)
    \end{align*}
    which resides, as a linear combination of kernel sections $k_\X(x_j^{(i)},\cdot)$, in the pre-defined RKHS $\mathcal{H}_{k_\X}$ (as, e.g., a Gaussian RKHS in~\cite{blanchard2021domain}).
\end{remark}

\subsection{Finite Sample Error Bound} \label{subsec:error_rates_o_notation}

For our algorithm defined in Subsection~\ref{subsec:algorithm}, we are able to provide finite sample bounds under certain assumptions, which can be essentially summarized by four categories: Classical regularity conditions on the involved kernels (mean embeddings, domain-specific regression, operator slope), classical assumptions on the effective dimension of the involved integral operators (domain-specific regression, operator slope learning), new assumptions on the (functional) data generating process, and new assumptions relating the domain-specific regression problems with the global functional regression.

Under these assumptions, it holds with probability at least $1-\delta$ that
\begin{align} \label{eq:error_rates_o_notation}
\mathcal{E}^\infty(g_{\z^{(1)},\ldots,\z^{(N)}}^{\lambda_1,\ldots,\lambda_N,\lambda})- \inf_{g:\M(\X)\to\left\{f:\X\to\mathbb{R}\right\}}\mathcal{E}^{\infty}(g) \leq \frac{\log \frac{4}{\delta}}{\delta^2} \mathcal{O}(N^{-\frac{1}{1+c_6}})\left(\mathcal{O}({n^{-\frac{1}{1+c_3}}})+\mathcal{O}(1)\right)
\end{align}
for some $0 \le c_3, c_6 \le 1$ independent of $P\in\M(\X\times\Y)$ and regularization parameter choices
$$\lambda=N^{-\frac{1}{1+c_6}}, \lambda_1=\cdots=\lambda_N=n^{-\frac{1}{1+c_3}}.$$

\begin{remark}
    Our finite sample bound in Eq.~\eqref{eq:error_rates_o_notation} accounts for properties of the RKHS $\mathcal{H}_k$ in which the slope $\beta\in\mathcal{H}_k$ is assumed to reside, see (A5) in Subsection~\ref{subsec:assumptions} below.
    However, it does not take into account the smoothness of $\beta$, which can be done by combining it with the methods from~\cite{mollenhauer2022learning}.
\end{remark}

\section{Finite Sample Error Bound}

In this Section, we detail our assumptions and the strategy for proving Eq.~\eqref{eq:error_rates_o_notation}.
All proofs can be found in Section~\ref{app:proofs}.
We start by introducing some further notation in Subsection~\ref{subsec:notation}.
Then, in Subsection~\ref{subsec:assumptions}, we summarize all assumptions, split into classical assumptions and some new assumptions.
In Subsection~\ref{subsec:preliminary_statements}, we prove some preliminary statements, which we use in Subsection~\ref{subsec:convergence_rates_result} to prove Eq.~\eqref{eq:error_rates_o_notation}.

\subsection{Notation}
\label{subsec:notation}

For an $n$-sized sample $\z:=(x_j,y_j)_{j=1}^n$ independently drawn according to some $P\in\M(\X\times\Y)$, we denote by $P^n(\z):=\bigotimes_{j=1}^n P(x_j,y_j)$.
We further denote by $\text{Supp}(E)\subseteq \M(\X\times\Y)$ the support of $E$.
For a bounded, compact and self adjoint operator $K$ on $L^2(P)$ with eigenvalues $(\theta_j)_{j=1}^{\infty}$, we denote the \textit{effective dimension} by
$$
\gamma_K(\lambda):=\text{Tr}((K+\lambda I)^{-1}K)=\sum_{j=1}^\infty \frac{\theta_j}{\lambda+\theta_j},
$$
see~\cite{caponnetto2007optimal}.
Let us also denote the \textit{covariance kernel} related to the sampling process of the empirical mean embeddings by
\begin{align*}
C(s,t)=\int_{\M(\X\times\Y)} \int_{\X^n} m_{\x}(s)m_{\x}(t) \diff P_\X^n(\x) \diff E(P),
\end{align*}
and its associated integral operator by
$(G_C f)(\cdot):=\int_{\X} C(x,\cdot) f(x) \diff x$.
We also denote the operator $T_{k}:=G_{k}^{\frac12} G_C G_{k}^{\frac12}$ for $G_k f:=\int_\X k(\cdot,x) f(x)\diff x$. Functions of self adjoint operators (e.g. the square root) are defined via the spectral calculus.
The function $f_0 \in L^2(\X^2)$ is defined such that $\beta(t,.)=G_k^{\frac12}f_0(t,.)$ for $\beta\in\mathcal{H}_k$ as in Eq.~\eqref{eq:functional_regression_of_slope}, it is well defined if $k$ is universal.

\subsection{Assumptions} \label{subsec:assumptions}

Our assumptions can be grouped in two parts: The first part deals with assumptions that are used in related works.
We formulate them in a way such that they hold uniformly over all $P \in \text{Supp}(E)$.
The second part discusses assumptions specific for our setting.
 All enumerated constants are independent from $P\in \text{Supp}(E)$.
 
 \paragraph*{Assumptions from Related Works}
\begin{enumerate}[label=(A{{\arabic*}})]
\item \textit{Assumption on kernels:} All applied kernels $k:\X \times \X \to \mathbb{R}$ belong to a family $\mathcal{K}$ of continuous (on the compact $\X$) kernels and admit a uniform bound $\kappa^2:=\sup_{k \in \mathcal{K}} \sup_{x\in\X} |k(x,x)|$.
    \item \textit{Regularity conditions for domain-specific regression:} For every $P\in \text{Supp}(E)$ the corresponding regression function $f_P$ satisfies $f_{P}=G^{\frac12}_{k,P} g_P$ for some $k \in \mathcal{K}$, $g_P \in L^2(P)$ and $G_{k,P}(f) :=\int_\X f(x) k(\cdot, x)\diff P(x)$.
    Moreover, $\norm{g_P}_{L^2(P)} \le c_1$ for some $c_1>0$.
    \item \textit{Assumptions on effective dimensions for domain-specific regression:} There are $c_2>0,0<c_3\le 1$ such that for any $P \in \text{Supp}(E)$ and $\lambda>0$ and $k \in \mathcal{K}$, the effective dimension of $G_{k,P}$ satisfies $\gamma_{G_{k,P}}(\lambda) \le c_2 \lambda^{-c_3}$.
    \item \textit{Assumptions for functional regression:}
    The slope $\beta$ of $G$ in Eq.~\eqref{eq:integral_form_of_G} satisfies $\beta(x,.) \in \Hs_{k}$ with an universal kernel $k\in\mathcal{K}$ (which ensures $G_k (L^2(\X))=\Hs_k$) and admits a bound $\int_\X \norm{\beta(x,\cdot)}_{\Hs_{k}}^2 \diff x \le c_4$ for some $c_4>0$.
    \item \textit{Assumptions on effective dimensions for functional regression:} For $T_k$ as defined in Subsection~\ref{subsec:notation}, it holds that $\gamma_{T_k}(\lambda)$ satisfies $\gamma_{T_k}(\lambda) \le c_5 \lambda^{-c_6}$  for some $c_5>0,0<c_6\le 1$.
    \item \textit{Zero intercept:} It holds that $a_0\equiv 0$.
    \end{enumerate}

    \paragraph*{Our Assumptions}
    \begin{enumerate}
    [label=(B{{\arabic*}})]
    \item \textit{Relation between distributions:} There are $c_{*} ,c^{*}>0$ such that for all $P\in \text{Supp}(E)$, we have that $\frac{1}{c_{*}} \norm{f}_{L^2(P)} \le  \norm{f}_{L^2(\X)} \le c^{*} \norm{f}_{L^2(P)}$ for all $f \in L^2(P)$.
   \item \textit{Coercivity of operator $G_C$:} There exists $c_7>0$ such that $\norm{g}^2_{L^2(\X^2)}\leq c_7 \langle G_C g, g\rangle_{L^2(\X^2)}$ for all $g\in L^2(\X^2)$ with $G_C$ as defined in Subsection~\ref{subsec:notation}.
    \item \textit{Independence of estimation errors:} The distributions of the estimation errors $(f_{\z}^{\lambda}-f_{P})$ and $G(m_{\x}-m_{P})$ are for any $\z=(\x,y)$ drawn from $P$ (drawn from $E$) independent from the distribution of $m_{\x}$.
    \item \textit{Estimation errors are unbiased:} The estimation biases satisfy
    \begin{align}
    \label{eq:estimation_bias_zero_assumption}
        &\int_{\M(\X\times\Y)} \int_{\X^n} G\cdot (m_{\x}-m_{P_\X}) \diff P_\X^n(\x) \diff E(P)\equiv 0\\
        &\int_{\M(\X\times\Y)} \int_{(\X\times\Y)^n} (f_{\z}^{\lambda}-f_{P}) \diff P^n(\z) \diff E(P)\equiv 0.
    \end{align}
\end{enumerate}

\begin{remark}    
    Assumptions (B1) and (B2) essentially relate differences in the $L^2$-spaces caused by drawing different distributions.
    Assumption (B3) allows easy separation of noise expectation from data expectations. (B3) can be relaxed using techniques from~\cite{jin2022minimax}, where the data generating model in Eq.~\eqref{eq:linear_operator_ansatz} is assumed without the independence assumption.
    Assumption (B4) can also be slightly relaxed by assuming zero bias conditioned at the draw of $P$~\cite{mollenhauer2022learning}.
    The assumptions (B1)--(B4) do not aim at
     an entirely exhaustive theoretical setup; but aim to lay the groundwork for new algorithms and analyses of domain generalization by functional regression.
\end{remark}

\subsection{Preliminaries}
\label{subsec:preliminary_statements}

Our finite sample bound relies on bounds from~\cite{tong2022non} for functional regression, which requires to bound the variances of the errors caused by finite sample approximation of mean embeddings $m_P$ and regression functions $f_P$.
This Section summarizes corresponding smaller statements preparing the ground for our finite sample bound.
The proofs are deferred to Section~\ref{app:proofs}.

\begin{lemma} \label{lemma:uniform bound_mean_embeddings}
    The kernel mean embedding $m_{P'}$ of any $P'\in\M(\X)$ w.r.t.~a kernel $k\in\mathcal{K}$ is bounded in $L^2(P)$-norm for any $P\in\M(\X)$ by
    \begin{align}
       &\norm{m_{P'}}_{L^2(P)}\leq \kappa^4. \label{eq:uniform bound mean embeddings} 
    \end{align}
\end{lemma} 
\noindent
The next Lemma~\ref{lemma:variance_bound_for_ridge_regression} follows from~\cite[Theorem~2]{guo2017learning}.

\begin{lemma}
    \label{lemma:variance_bound_for_ridge_regression}
    Let $P\in\text{Supp}(E)$, $n \in \mathbb{N}$ and assume (A2)--(A3).
    Then, for $\lambda=n^{-\frac{1}{c_3+1} }$, we have that
    \begin{align}
    \label{eq:variance_bound_for_ridge_regression}
         \int_{(\X\times\Y)^{n}} \norm{f_{\z}^{\lambda}-f_{P}}_{L^2(\X)}^2 \diff P^n(\z) \leq (c^{*})^2 c_8 n^{-\frac{1}{c_3+1} },
    \end{align}
    for some $c_8>0$ that is independent from $P$.
\end{lemma}

\begin{lemma}[\cite{wolfer2022variance}, Section 2, Remark 2.1]
    \label{lemma:estimation_bound_mean_embeddings}
    For $P\in\M(\X\times\Y)$, $n \in \mathbb{N}$ and $k \in \mathcal{K}$, we have that
    \begin{align}
    \label{eq:estimation_bound_mean_embeddings}
        \int_{\X^n} \norm{m_{\x}-m_{P_\X}}_{\Hs_{k}}^2 \diff P_\X^n(\x) \leq \frac{\kappa^2}{n}.
    \end{align}
\end{lemma}

\noindent
Lemma~\ref{lemma:estimation_bound_mean_embeddings} leads to the following variance bound.

\begin{lemma}
    \label{lemma:variance_mapped_mean_embedding}
  For $P\in\M(\X\times\Y)$, $n \in \mathbb{N}$ and $k \in \mathcal{K}$, under assumption (A4), we have that
    \begin{align}
        \label{eq:variance_mapped_mean_embedding}
        \int_{\X^n}
        \norm{G\cdot(m_{\x}-m_{P_\X})}_{L^2(\X)}^2
        \diff P_\X^n(\x) \leq \frac{c_4 \kappa^6}{n}.
    \end{align}
\end{lemma}

\subsection{Convergence Rates Result}
\label{subsec:convergence_rates_result}

Now we continue our investigations concerning Eq.~\eqref{eq:error_rates_o_notation}.
We need to analyze the difference 
\begin{align}\label{eq:error_decomp_1}
\mathcal{E}^\infty(g_{\z^{(1)},\ldots,\z^{(N)}}^{\lambda_1,\ldots,\lambda_N,\lambda}) &- \inf_{g:\M(\X)\to\left\{f:\X\to\mathbb{R}\right\}}\mathcal{E}^{\infty}(g)=\nonumber 
\\= &\int_{\M(\X)} \int_{\X\times\Y} \left(G_{\z^{(1)},\ldots,\z^{(N)}}^{\lambda_1,\ldots,\lambda_N,\lambda} m_{P_\X}(x)-y\right)^2\diff P(x,y) \diff E(P)\nonumber
\\ &\quad\quad\quad-\int_{\M(\X)} \int_{\X\times\Y} \left(f_{P}(x)-y\right)^2\diff P(x,y) \diff E(P) \nonumber \\
= &\int_{\M(\X)}  \norm{G_{\z^{(1)},\ldots,\z^{(N)}}^{\lambda_1,\ldots,\lambda_N,\lambda} m_{P_\X}(x)-f_{P}(x)}_{L^2(P)}^2 \diff P(\x) \diff E(P)\nonumber \\
=&\int_{\M(\X)} \int_{\X} \norm{G_{\z^{(1)},\ldots,\z^{(N)}}^{\lambda_1,\ldots,\lambda_N,\lambda} m_{P_\X}-G m_{P_\X}}_{L^2(P)}^2  \diff P(\x) \diff E(P),
\end{align}
where the second equality follows from the bias-variance decomposition (see e.g. \citet[Proposition~1]{cucker2002mathematical}) and the last equality from our linear operator Ansatz in Subsection~\ref{subsec:linear_ansatz}.
In order to further analyze Eq.~\eqref{eq:error_decomp_1}, we use assumption (B1) from Subsection~\ref{subsec:assumptions} and obtain
\begin{align} \label{eq:error_decomp_2}
(c_{*})^2 \norm{G_{\z^{(1)},\ldots,\z^{(N)}}^{\lambda_1,\ldots,\lambda_N,\lambda}-G}^2_{\text{Op}(L^2(\X))}
  \norm{m_{P_\X}}_{L^2(\X)}^2,
\end{align}
where $\norm{\cdot}_{\text{Op}(L^2(\X))}$ denotes the operator norm on $L^2(\X)$.
As $\norm{m_{P_\X}}_{L^2(\X)}$ can be uniformly bounded (using Lemma~\ref{lemma:uniform bound_mean_embeddings}), we only need to care about $\norm{G_{\z^{(1)},\ldots,\z^{(N)}}^{\lambda_1,\ldots,\lambda_N,\lambda}-G}^2_{\text{Op}(L^2(\X))}$. This Hilbert-Schmidt norm relates to the $L^2(\X)$-norm of the difference between the corresponding slope functions, which is used by the following key lemma, together with assumption (B2) and methods from~\cite{tong2022non}.

\begin{lemma}
\label{lemma:operator_norm}
Consider the algorithm introduced in Subsection \ref{subsec:algorithm}.
Under the assumptions stated in Subsection \ref{subsec:assumptions}, if we set $\lambda=N^{-\frac{1}{1+c_{6}}}$ and $\lambda_i=n^{-\frac{1}{1+c_{3}}}$ for $i=1,...,N$, we have that for any $0<\delta<1$ with probability $1-\delta$: 
\begin{align*}
\norm{G_{\z^{(1)},\ldots,\z^{(N)}}^{\lambda_1,\ldots,\lambda_N,\lambda}-G}^2_{\mathrm{Op}(L^2(\X))} \le c_7 C(\Bar{\sigma}^2)\frac{\log \frac{4}{\delta}}{\delta^2} N^{-\frac{1}{1+c_6}},
\end{align*}
where 
 \begin{align*}
 C(\Bar{\sigma}^2)&=2 \left(\frac{\Bar{\sigma}^2\left(2 \kappa^5   \left(\kappa^5  +\sqrt{c_5}\right)+1\right)^6}{\kappa^{10} }+\left(2 \kappa^5  \left(\kappa^5 +\sqrt{c_5}\right)+1\right)^2 \norm{f_0}_{L^2(\X^2)}^2\right),
 \end{align*}
 and
\begin{align} \label{eq:bound_variance_total}
\Bar{\sigma}^2 = (c^{*})^2 c_8 n^{-\frac{1}{c_3+1} }+ c_4 \kappa^6 n^{-1} +\sigma^2.
\end{align}
\end{lemma}
Applying Lemma~\ref{lemma:operator_norm} to Eq.~\eqref{eq:error_decomp_2}, and combining it with the variance bounds in lemma~\ref{lemma:variance_bound_for_ridge_regression} and Lemma~\ref{lemma:variance_mapped_mean_embedding} results in our main finite sample bound.

\begin{theorem}
Consider the algorithm introduced in Subsection \ref{subsec:algorithm}.
Under the assumptions stated in Subsection \ref{subsec:assumptions}, if we set $\lambda=N^{-\frac{1}{1+c_{6}}}$ and $\lambda_i=n^{-\frac{1}{1+c_{3}}}$ for $i=1,...,N$, we have that for any $0<\delta<1$ with probability $1-\delta$:
\begin{align} \label{eq:main_bound}
\mathcal{E}^\infty(g_{\z^{(1)},\ldots,\z^{(N)}}^{\lambda_1,\ldots,\lambda_N,\lambda})- \inf_{g:\M(\X)\to\left\{f:\X\to\mathbb{R}\right\}}\mathcal{E}^{\infty}(g) \le C'(n) \frac{\log \frac{4}{\delta}}{\delta^2} N^{-\frac{1}{1+c_6}},
\end{align}
where
\begin{align}
\label{eq:main_constant}
C'(n)=c_{*}^2  \kappa^8 C(\Bar{\sigma}^2)=~&\frac{2\kappa^8 c_{*}^2 c_7 ((c^{*})^2 c_8 n^{-\frac{1}{c_3+1} }+ c_4 \kappa^6 n^{-1} +\sigma^2 )(2 \kappa^5   (\kappa^5  +\sqrt{c_5})+1)^6}{\kappa^{10} }\nonumber\\
&+2\kappa^8 c_{*}^2 c_7 (2 \kappa^5  (\kappa^5 +\sqrt{c_5})+1)^2 \norm{f_0}_{L^2(\X^2)}^2.
\end{align}
\end{theorem}

\section{Numerical Example}

The goal of our numerical example is (a) to underpin the potential of the proposed functional regression approach for domain generalization, and, (b) to illustrate an implementation of the algorithm in Subsection~\ref{subsec:algorithm}, which we also provide in python.

\paragraph*{Data Generation}
We generated $N=100$ input source samples $(\x^{(i)})_{i=1}^N\in (\X)^{n\times N}$, $\X=[0,1]$, each of size $n=100$, drawn independently from $N$ different truncated Normal distributions $(P_\X^{(i)})_{i=1}^N$, see Eq.~\eqref{eq:source_samples}.
The values of the truncated Normal distributions are generated to lie in the compact interval $[\mu^{(i)}-0.3, \mu^{(i)}+0.3]$ with means $\mu^{(i)}:=\int_\X x\diff P_\X^{(i)}(x)$ generated independent and uniformly distributed in the interval $[0.3,0.7]$ and variances generated in the interval $[0.025, 0.125]$, see Figure~\ref{fig:a}.
The outputs $(\y^{(i)})_{i=1}^N\in (\mathbb{R})^{n\times N}$, $\Y=\mathbb{R}$, corresponding to the inputs $(\x^{(i)})_{i=1}^N\in \mathbb{R}^{n\times N}$, are generated according to the equation
\begin{align*}
y_j^{(j)} = \frac{1}{10} \sin\!\left(\frac{3 x_j^{(i)}}{(\mu^{(i)})^2}\right) + \frac{9}{10} - \left(1.7 \left(x_j^{(i)}-\frac{1}{2}\right)\right)^2 + \epsilon_j^{(i)},
\end{align*}
where $(\epsilon_j^{(i)})_{i=1}^N$ are independently drawn from a Normal distribution with mean $0$ and variance $0.02$, see Figure~\ref{fig:b}.

\paragraph*{State-of-the-art Baselines}
Recall the goal of domain generalization to learn, from the source samples $(\x,\y)_{i=1}^N$ a model $g:\M(\X)\to\{f:\X\to\Y\}$ that performs well on data from new \textit{target} distributions.
In our example, this means, that $g$ needs to be computed from the data $\left(\x^{(i)},\y^{(i)}\right)_{i=1}^N$ described above.
Figure~\ref{fig:c} shows as a dashed line the prediction of a single ridge regression Eq.~\eqref{eq:domain_specific_ridge_regression} on the \textit{pooled} data $\left(x_j^{(i)},y_j^{(i)}\right)_{i\in\{1,\ldots,N\},j\in\{1,\ldots,n\}}$, which serves as the baseline representing the state of the art.
We also implemented the approach in~\cite{blanchard2021domain}, but it was not able to outperform the pooling procedure, although an intensive parameter search was performed as follows.

\sloppy
Following~\cite{blanchard2021domain}, the parameter $\lambda$ and the kernel of the RKHS of the ridge regression for pooling were chosen by $5$-fold cross-validation on a grid of values, $\lambda\in\{10^{-1},10^{-2},\ldots,10^{-6}\}$ and the kernel either as Gaussian kernel $k(x,y)=e^{-\frac{|x-y|^2}{2 l^2}}$ with $l\in\{1,5,10\}$ or as periodic kernel $k(x,y)=e^{-\frac{2 \sin^2(\pi |x-y|/p)}{l^2}}$ with $l\in\{1,10^{-1},10^{-2}\}, p=1$.
We followed the same procedure for choosing the parameters and the kernels used in our implementation of the marginal transfer approach of~\cite{blanchard2021domain} in Subsection~\ref{subsec:marginal_transfer_learning}.
In particular, we chose the kernel for computing the applied empirical kernel mean embeddings as Gaussian kernel with $l\in\{10^{-2},10^{-3}\}$ or the periodic kernel with $l=1, p=1$, we choose the kernel $k_\X$ either as the Gaussian kernel with $l\in\{10^{3},10^2,\ldots,10^{-3}\}$ or the periodic kernel with $l\in\{1,10^{-1},10^{-2}\}, p=1$, and the kernel between mean embeddings as the Gaussian kernel with $l\in\{1,10^3,10^{-3}\}$.
The regularization parameter was chosen as $\lambda\in\{10^{3},10^{2},\ldots,10^{-4}\}$.



 \begin{figure}
\centering
\subfigure[input distributions]{%
\resizebox*{5cm}{!}{\includegraphics{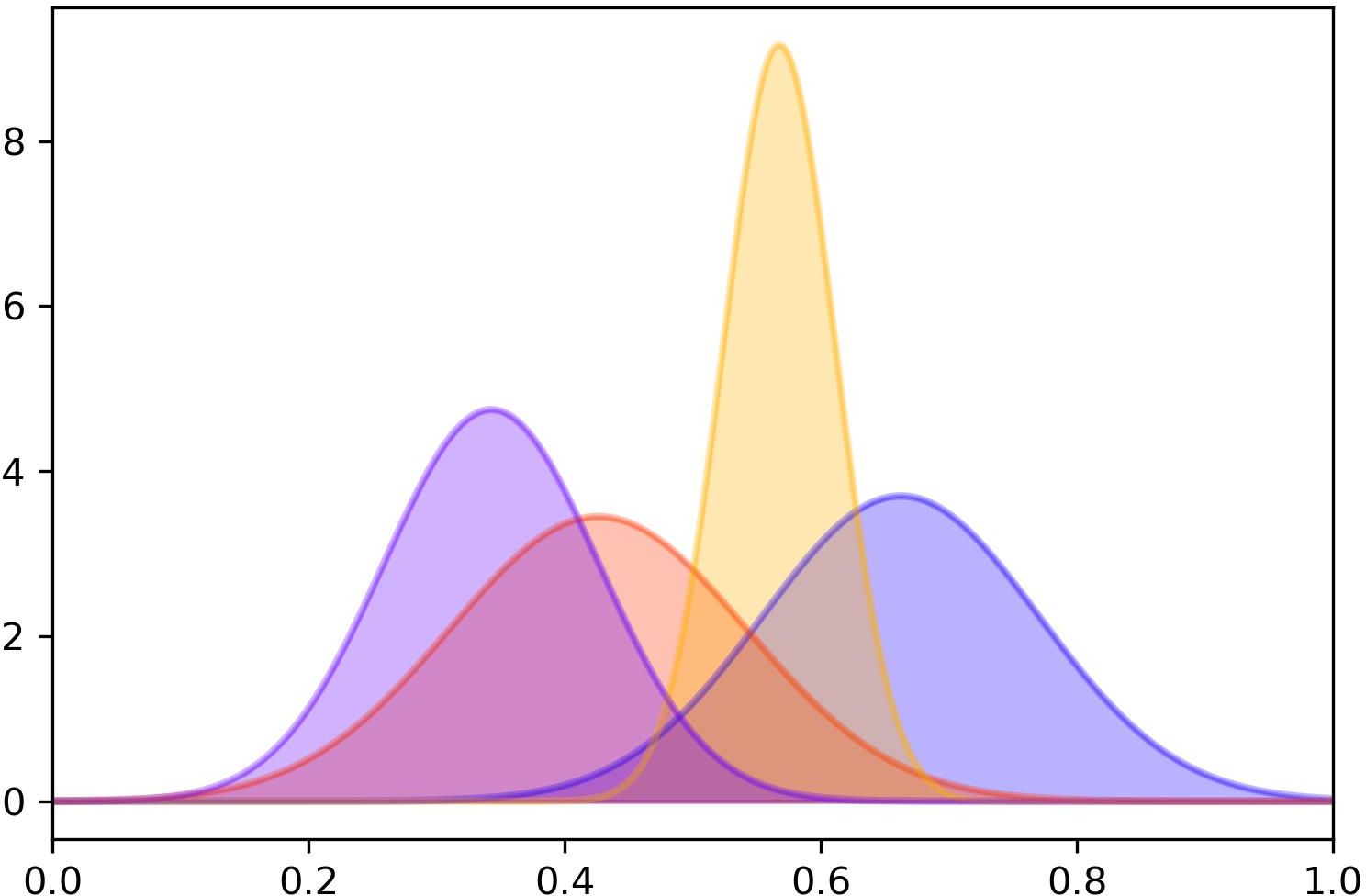}\label{fig:a}}}\hspace{5pt}
\subfigure[regression functions]{%
\resizebox*{5cm}{!}{\includegraphics{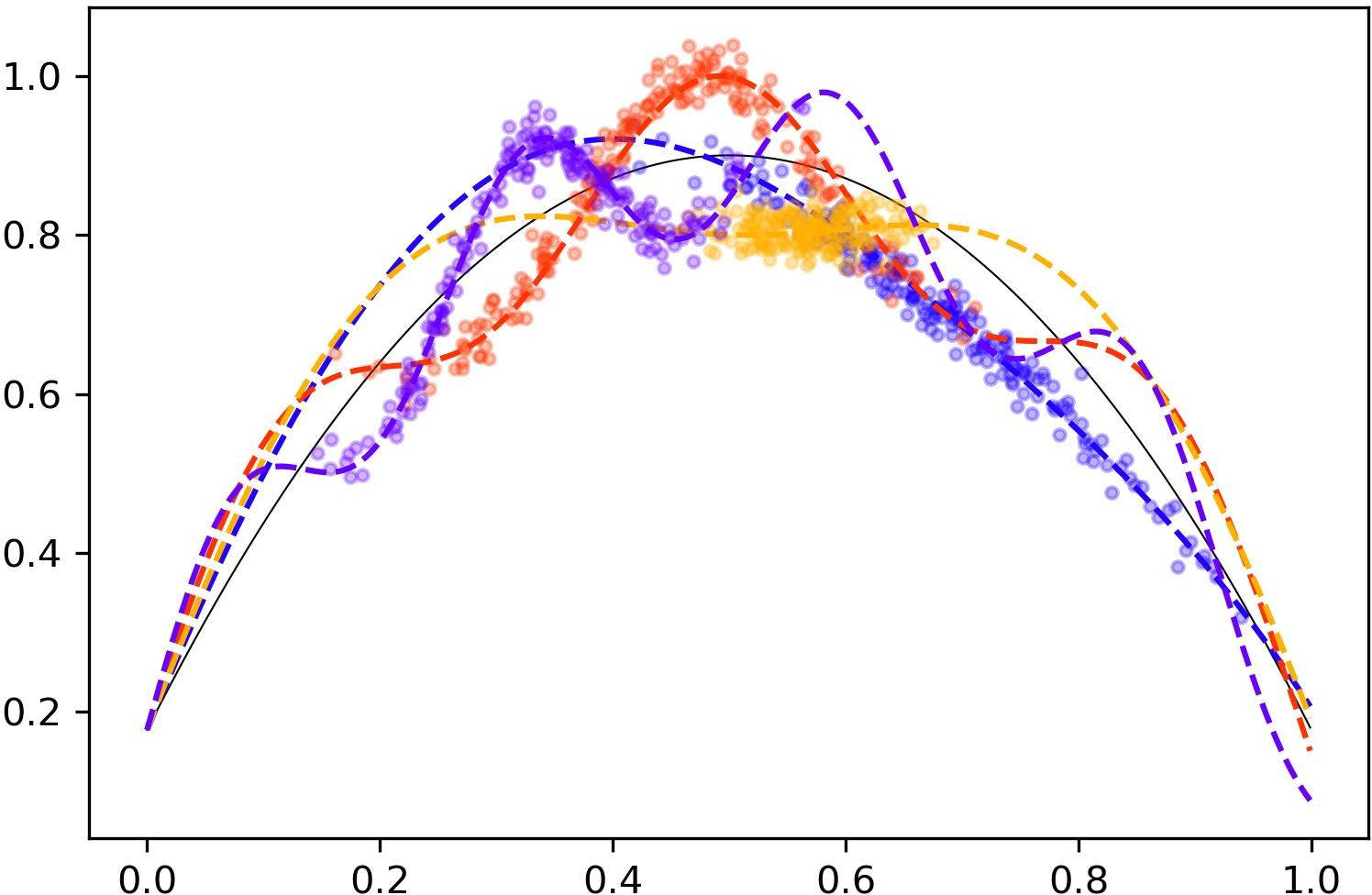}\label{fig:b}}}
\subfigure[regression on pooled data]{%
\resizebox*{5cm}{!}{\includegraphics{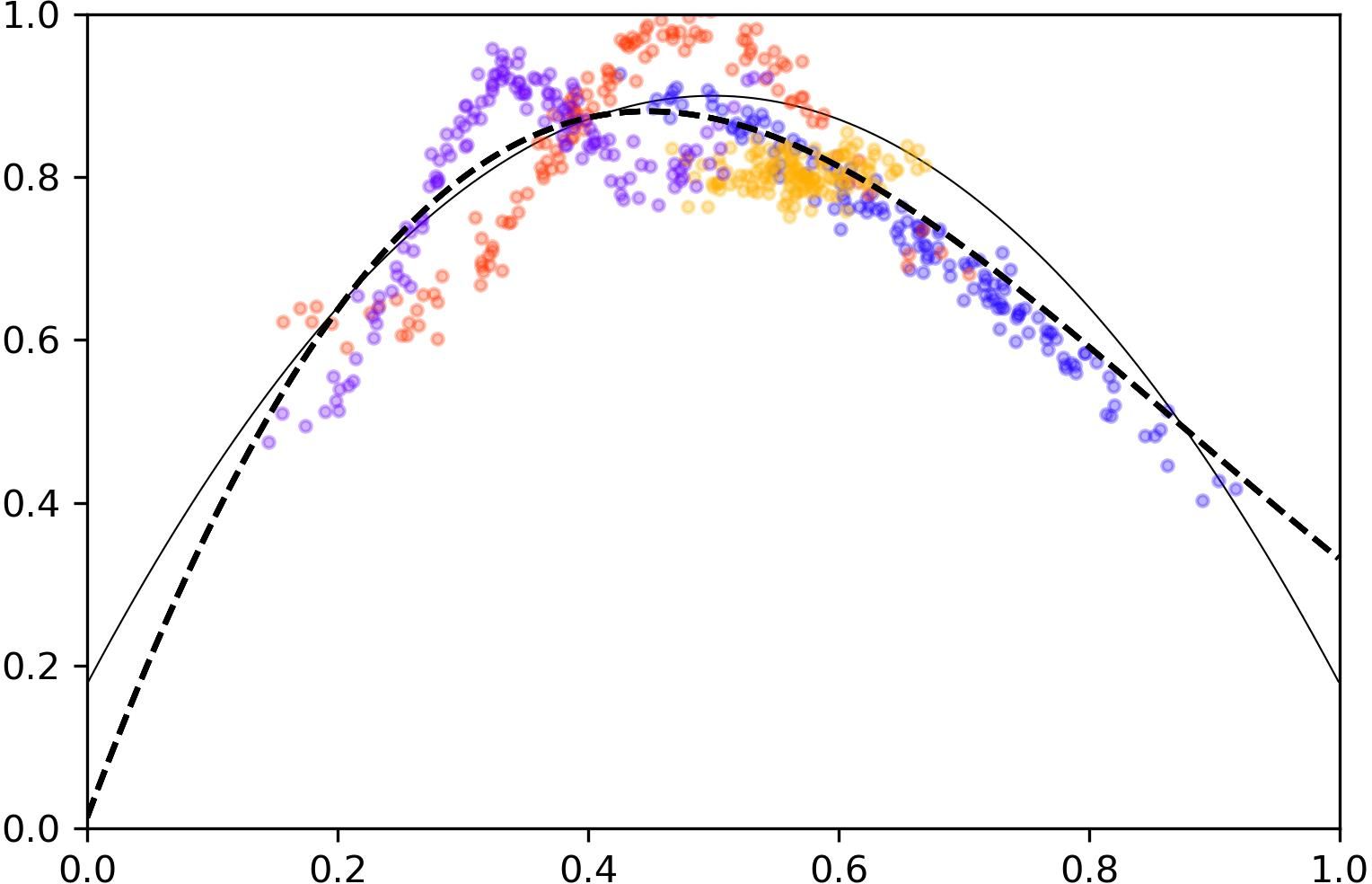}}\label{fig:c}}\hspace{5pt}
\subfigure[our predictions]{%
\resizebox*{5cm}{!}{\includegraphics{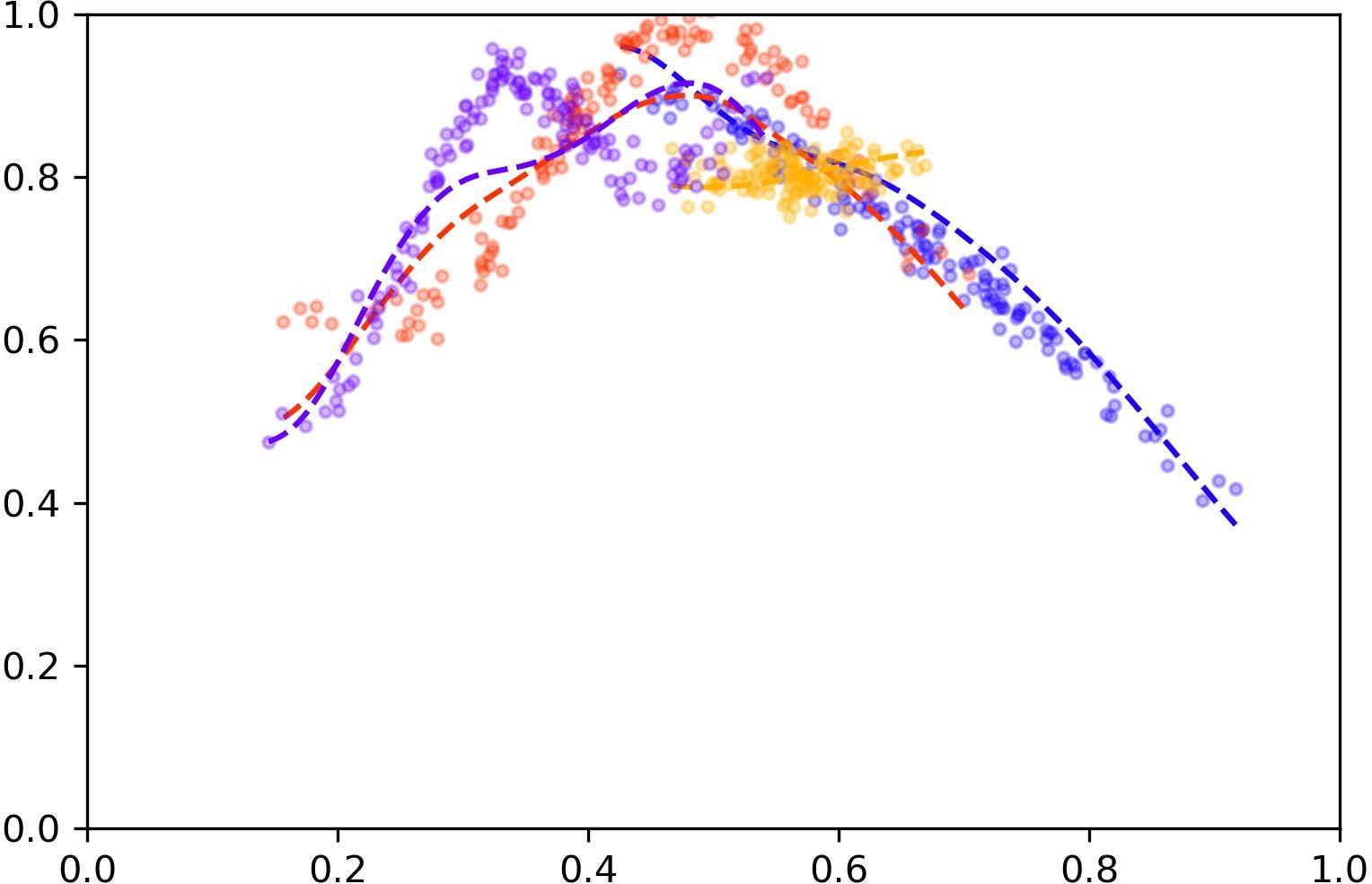}}
\label{fig:d}}
\caption{Our approach maps kernel mean embeddings of input distributions (a) to regression functions (b, dashed) and allows to outperform ridge regression on pooled data (c, dashed), in contrast to~\cite{blanchard2021domain} (also c, dashed), which is illustrated by four random test predictions of our approach (d, dashed).}
 \label{fig:numerical_illustration}
\end{figure}

\paragraph*{Implementation of Functional Regression Approach}
The implementation of our general functional regression algorithm described in Subsection~\ref{subsec:algorithm} has two main steps.

The first step is ridge regression on each source distribution $P^{(i)}$ to compute an estimator $f_{\z^{(i)}}^{\lambda_i}$.
In this step, the full potential of our approach can be seen, as for each $i\in\{1,\ldots,N\}$, a different RKHS can be learned using cross-validation with $\lambda_i\in\{10^{-1},\ldots,10^{-4}\}$ and kernels being either a Gaussian kernel with $l\in\{10^{-1},\ldots,10^{-4}\}$ or a periodic kernel with $l\in\{1,10^{-1},10^{-2}\}$ and $p\in\{1,2,3,5,10\}$.

The second step is penalized estimation according to Eq.~\eqref{eq:functional_regression_of_slope}. In this step, we follow step-by-step the Algorithm~1 in Section~4 of~\cite{tong2022non}.
This algorithm is essentially ridge regression, but it requires to estimate functions instead of scalar weights for the kernel sections in the solution granted by the representer theorem, see e.g.~\cite[Eq.~(16)]{scholkopf2001generalized}.
One particular difficulty which has to be mentioned at this point is the estimation of the $L^2([0,1])$-norm, which is required in this step.
This is done in~\cite{tong2022non} by discretization of the interval $[0,1]$, which is simple in our numerical example using $1000$ equally distributed grid points, but it suffers from the curse of dimensionality when the input data dimension increases.

\paragraph*{Result}
Figure~\ref{fig:d} shows some predictions (dashed lines) of our algorithm on new target distributions (not included in the source distributions).
Although the predictions are not perfect, they clearly outperform the dashed line of the simple pooling approach (and also the algorithm of~\cite{blanchard2021domain}).
We also measured the difference in empirical least-squares test error.
On average (over $20$ new test distributions),
the dashed regression functions illustrated in Figure~\ref{fig:b} (lower bound) achieve an error of $0.0008$, the gray parabola in Figure~\ref{fig:c} (upper bound) has an error of $0.0062$, and, the simple pooling approach (dashed in Figure~\ref{fig:c}) and marginal transfer learning don't achieve an error lower than $0.0042$.
Our implementation achieves the lowest error of $0.0029$.


\section{Conclusion and Future Work}

In this work, we study domain generalization as a problem of functional regression, i.e., regression with functional input and output, by directly learning the relationship between domain-specific marginal distributions of inputs and corresponding conditional distributions of outputs given inputs.
Our new conceptualization leads to an operator learning algorithm with finite sample bounds.
We also provide numerical illustrations showing its advantage of explicit computation of domain-specific predictors in possibly different reproducing kernel Hilbert spaces.

Our work aims at setting the ground for (to the best of our knowledge first) finite sample bounds for domain generalization.
However, it leaves an entirely exhaustive statistical analysis for future research, e.g., by taking the smoothness of the operator slope and more general non-linear functional regression, into account.

\section{Proofs}
\label{app:proofs}

\begin{proof}[Proof of Lemma \ref{lemma:uniform bound_mean_embeddings}]
It holds that
\begin{align*}
\norm{m_{P'}}_{L^2(P)}^2 &=\int_{\X} \left(\int_\X k(x,x')\diff P'(x')\right)^2 \diff P(x)\\
&\le \int_{\X} \left(\int_\X \left|k(x,x')\right|\diff P'(x')\right)^2 \diff P(x) \\
&\leq \int_{\X} \left(\int_\X |k(x,x)|\, |k(x',x')|\diff P'(x')\right)^2 \diff P(x)\\
&\le \left( \int_\X \sup_{x \in \X} \left|k(x,x)\right|^2 \diff P(x) \right)^2 \le \kappa^8.
\end{align*}

\end{proof}

\begin{proof}[Proof of Lemma \ref{lemma:variance_bound_for_ridge_regression}]
\label{proof:variance_bound_for_rigde_regeression}
From (B1) in Subsection \ref{subsec:assumptions}, it follows that
\begin{align*}
\norm{f_{\z}^{\lambda}-f_{P}}_{L^2(\X)}^2 \le (c^{*})^2 \norm{f_{\z}^{\lambda}-f_{P}}_{L^2(P)}^2.
\end{align*}
Next, we apply~\cite[Eq.~(12) in Theorem~2]{guo2017learning} with $r=\frac12$ and the case of Tikhonov regularization for $\lambda=n^{-\frac{1}{c_3+1}}$.
In particular, it states that, under assumptions (A2) and (A3),
\begin{align*}
         \int_{(\X\times\Y)^{n}} \norm{f_{\z}^{\lambda}-f_{P}}_{L^2(P)}^2 \diff P^n(\z) \leq c_P n^{-\frac{1}{c_3+1} }, 
\end{align*}
where $c_P$ is some constant that may depend on $P$.
It remains to ensure that $c_P$ does not depend on $P$.
The explicit value of the constant (for a possibly more general setting), can be found in the last line of the proof of Theorem 2 in \citet{guo2017learning}, however, we only need the case $r=\frac12$ and Tikhonov regularization (which leads to $b=\nu_g=\gamma_{\nu}=1$ in their paper).
Combining the notations of \citet{guo2017learning} with our assumptions, the constant reads as follows:
\begin{align*}
c_P=(6 \Gamma (9)+\log (6^8))&\left(\frac{8 \sup_{y\in \Y} |y|}{\kappa}+6 \norm{g_P}_{L^2(P)}\right)^2\cdot\\
&\cdot\left[4\left(\kappa^2+\kappa \sqrt{c_2}\right)^2+1\right]^2\left(2 \kappa^2+2 \kappa \sqrt{c_2}+1\right)^4,
\end{align*}
where $\Gamma(.)$ denotes the Gamma function.
Using assumption (A2), we can uniformly bound $\norm{g_P}_{L^2(P)}$ and get
\begin{align*}
    c_P \le (6 \Gamma (9)+\log (6^8))&\left(\frac{8 \sup_{y\in \Y} |y|}{\kappa}+6 c_1 \right)^2\cdot\\
    &\cdot\left[4\left(\kappa^2+\kappa \sqrt{c_2}\right)^2+1\right]^2\left(2 \kappa^2+2 \kappa \sqrt{c_2}+1\right)^4=:c_8
\end{align*}
which is independent of $P$.

\end{proof}

\begin{proof}[Proof of Lemma \ref{lemma:variance_mapped_mean_embedding}]
\label{proof:variance_mapped_mean_embedding}
It holds that
    \begin{align*}
        &\int_{\X^n}
        \norm{G\cdot(m_{\x}-m_{P_\X})}_{L^2(\X)}^2
        \diff P_\X^n(\x)\\
        &=\int_{\X^n}
        \norm{\int_\X (m_{\x}-m_{P_\X})(x) \beta(\cdot,x)\diff x}_{L^2(\X)}^2
        \diff P_\X^n(\x)\\
        &\le \int_{\X^n}
        \norm{\norm{m_{\x}-m_{P_\X}}_{L^2(\X)} \norm{\beta(y,x)}_{L_x^2(\X)}}_{L_y^2(\X)}^2
        \diff P_\X^n(\x)\\
        &\le \int_{\X^n}
        \norm{m_{\x}-m_{P_\X}}_{L^2(\X)}^2 \int_\X\norm{\beta(y,x)}_{L_x^2(\X)}^2 \diff y
        \diff P_\X^n(\x)\\
        &\leq \int_{\X^n}
        \norm{m_{\x}-m_{P_\X}}_{L^2(\X)}^2 \int_\X\kappa^2 \norm{\beta(y,\cdot)}_{\Hs_k}^2 \diff y
        \diff P_\X^n(\x)\\
        &\leq c_4 \kappa^2 \int_{\M(\X\times\Y)}
        \int_{\X^n}
        \norm{m_{\x}-m_{P_\X}}_{L^2(\X)}^2  \diff P_\X^n(\x)\\
        &\leq c_4\kappa^2 \int_{\M(\X\times\Y)}
        \int_{\X^n}
        \kappa^2 \norm{m_{\x}-m_{P_\X}}_{\Hs_{k}}^2  \diff P_\X^n(\x)\\
        &\leq c_4 \frac{\kappa^6 }{n},
    \end{align*}
    where the first inequality is Cauchy-Schwartz and the last inequality follows from Lemma \ref{lemma:estimation_bound_mean_embeddings}.
\end{proof}
\noindent
In order to prove Lemma~\ref{lemma:operator_norm}, we will discuss the following Corollary first, which applies the main arguments from~\cite{tong2022non} to our setting.
For readers convenience, we provide an overview here:
\begin{corollary}
\label{cor:func_reg}
Consider the algorithm introduced in Subsection \ref{subsec:algorithm}. Under the assumptions stated in Subsection \ref{subsec:assumptions}, if we set $\lambda=N^{-\frac{1}{1+c_6}}$ and $\lambda_i=n^{-\frac{1}{1+c_{3}}}$, we have that for any $0<\delta<1$ with probability $1-\delta$: 
\begin{align*}
\int_{\M(\X)}  \int_{\X^n} \norm{(G_{\z^{(1)},\ldots,\z^{(N)}}^{\lambda_1,\ldots,\lambda_N,\lambda}-G) m_{\x}}_{L^2(\X)}^2 \diff P_\X^n(\x) \diff E(P) \le C(\Bar{\sigma}^2)\frac{\log \frac{4}{\delta}}{\delta^2} N^{-\frac{1}{1+c_6}}.
\end{align*}
\end{corollary}

\begin{proof}[Proof of Corollary~\ref{cor:func_reg}]

\textbf{Step 1:}
Let us start by reformulating our linear operator Ansatz \ref{subsec:linear_ansatz} as follows:
\begin{equation} \label{eq:func_reg_problem}
f_{\z}^{\lambda}=G \cdot m_{\x}+ \underbrace{f_{\z}^{\lambda}-f_{P}+G\cdot m_{P_\X} -G \cdot m_{\x}+\varepsilon}_{:=\Bar{\varepsilon}}.
\end{equation}
We further decompose the noise term as $\Bar{\varepsilon}=\varepsilon_1+\varepsilon_2+\varepsilon$ with
\begin{equation}\label{eq:error_mean_embedding}
    \varepsilon_1=f_{\z}^{\lambda}-f_{P}
\end{equation}
and
\begin{equation}\label{eq:error_kernel_regression}
    \varepsilon_2=G \cdot (m_{P_\X}-m_{\x}).
\end{equation}
Assumption (B3) ensures that $\Bar{\varepsilon}$ is independent from the sampling process of $m_\x$.
(B4) ensures that both $\varepsilon_1$ and $\varepsilon_2$ are unbiased.
Now we take care of the associated variances $\sigma_1^2$ and $\sigma_2^2$ of $\varepsilon_1$ and $\varepsilon_2$, respectively (this is where the data sample size $n$ enters the bound).
In particular, $\varepsilon_1$ is handled by Lemma \ref{lemma:variance_bound_for_ridge_regression} since
\begin{align*}
     \sigma_1^2=\int_{\M(\X\times \Y)}\int_{(\X\times\Y)^{n}} \norm{f_{\z}^{\lambda}-f_{P}}_{L^2(P)}^2 \diff P^n(\z) \diff E(P)
\end{align*}
and $\varepsilon_2$ is handled by Lemma \ref{lemma:variance_mapped_mean_embedding}.
Using both lemmas together, we end up with Eq.~\eqref{eq:bound_variance_total}.

\textbf{Step 2:}
From assumption (A4) we get $G_k (L^2(\X))=\Hs_k$.
That is, there are $f_0, f_{N,\lambda} \in L^2(\X^2)$ such that 
\begin{align} \label{eq:def_f_fN}
\beta_{\z^{(1)},\ldots,\z^{(N)}}^{\lambda_1,\ldots,\lambda_N,\lambda}(t,.)=G_k^{\frac12}f_{N,\lambda}(t,.) \qquad \text{ and } \qquad \beta(t,.)=G_k^{\frac12}f_0(t,.)
\end{align}
for all $t\in \X$. Using the definition of the operators and Fubini's theorem, it follows that 
    \begin{align}
        \int_{\M(\X)}  \int_{\X^n} \norm{(G_{\z^{(1)},\ldots,\z^{(N)}}^{\lambda_1,\ldots,\lambda_N,\lambda}-G) m_{\x}}_{L^2(\X)}^2 \diff P_\X^n(\x) \diff E(P)&=\innerpro{f_0-f_{N,\lambda}, T_k(f_0-f_{N,\lambda})} \nonumber \\
        &=\norm{T_k^{\frac12}(f_0-f_{N,\lambda})}^2_{L^2(\X^2)}, \label{eq:mat_sin_main_1}
    \end{align}
    where for $f \in L^2(\X^2)$, $T_kf(s,t)=T_k(f(s,.))(t)$.

    \textbf{Step 3:} The empirical counterpart $C_N$ of the covariance kernel $C$ is defined by $C_N(s,t):=\frac1N \sum_{i=1}^N m_{\x_i}(s) m_{\x_i}(t)$ and we denote the associated integral operator by $G_{C_N}$. Moreover, $T_{k,N}=T_k^{\frac12} G_{C_N} T_k^{\frac12}$. Next, Eq.~\eqref{eq:functional_regression_of_slope} can be written as
    \begin{align}
         f_N^{\lambda}
        &=\argmin_{f \in L^2(\X^2)}
        \frac{1}{N}\sum_{i=1}^N
        \norm{G \cdot m_{\x_i}+\bar{\varepsilon}_i-\int f(\cdot, x')(G_k^{\frac12}m_{\x^{(i)}})(x')\diff x' }_{L^2(\X)}^2+\lambda \norm{f}_{L^2(\X^2)}^2 \nonumber \\
        &=\argmin_{f \in L^2(\X^2)}
        \frac{1}{N}\sum_{i=1}^N
        \norm{\underbrace{G \cdot m_{\x_i}+\bar{\varepsilon}_i}_{=Y_i}-\int_\X (G_k^{\frac12}f(\cdot, x')) m_{\x^{(i)}}(x')\diff x' }_{L^2(\X)}^2+\lambda \norm{f}_{L^2(\X^2)}^2.   \label{eq:main_obj_reformulation}   
    \end{align}
    It can be shown by similar techniques as in \citet[Theorem 5.1, Page 117]{engl1996regularization} that 
    \begin{align} \label{eq:mat_sin_main_2}
        f_N^{\lambda}=(T_{k,N}+\lambda I)^{-1}(T_{k,N} f_0+g_N),
    \end{align}
    where 
    \begin{align} \label{eq:def_gN}
    g_N(s,t)=\frac1N \sum_{i=1}^N \bar{\varepsilon}_i(s) (G_k^{\frac12} m_{\x^{(i)}})(t).
    \end{align}
    For the reader's convenience, we will also include a short proof of Eq.~\eqref{eq:mat_sin_main_2} in the following, which is mostly considered as standard in the literature on functional regression.
    
    Let us therefore define $\tilde{T}_k^i: L^2(\X^2) \to L^2(\X)$ by $(\tilde{T}_k^i f)(y)= \int_\X (G_k^{\frac12}f(y, x')) m_{\x^{(i)}}(x')\diff x'$.
    It is easy to see that $(\tilde{T}_k^i)^{*}  g(x,y)= (G_k^{\frac12}m_{\x^{(i)}})(x) g(y)$ and $T_{k,N}=\frac1N \sum_{i=1}^N (\tilde{T}_k^{i})^{*} \tilde{T}^i_k$. Eq.~\eqref{eq:main_obj_reformulation} can now be rewritten as $\argmin_{f \in L^2(\X^2)} \frac{1}{N}\sum_{i=1}^N \norm{Y_i-\tilde{T}_k^i f}^2_{L^2(\X)}+\lambda \norm{f}_{L^2(\X^2)}$.
    Next, we take the Frechét derivative of $F(f)=\norm{Y_i-\tilde{T}_k^i f}^2_{L^2(\X)}+\lambda \norm{f}_{L^2(\X^2)}$ in direction $v \in L^2(\X^2)$ and get 
    \begin{align}
    F'(f)(v)&=2\frac{1}{N}\sum_{i=1}^N \left( \innerpro{Y_i-\tilde{T}_k^i f,\tilde{T}^i_k v }_{L^2(\X)}+2\lambda \innerpro{f,v}_{L^2(\X^2)} \right) \nonumber \\
    &=2 \innerpro{\sum_{i=1}^N(\tilde{T}_k^i)^{*} Y_i+(T_{k,N}+\lambda)f,v}_{L^2(\X^2)}. \label{eq:frechet_result}
    \end{align}
Setting Eq.~\eqref{eq:frechet_result} to zero and using $Y_i=\tilde{T}_k^i \beta +\bar{\varepsilon_i}$, we end up with Eq.~\eqref{eq:mat_sin_main_2}, since all functionals are convex.

\textbf{Step 4:} Combining Eq.~\eqref{eq:mat_sin_main_1} and Eq.~\eqref{eq:mat_sin_main_2}, we obtain:
\begin{align}
\norm{T_k^{\frac12}(f_0-f_{N,\lambda})}_{L^2(\X^2)} \le &\underbrace{\norm{T_k^{\frac12}(T_{k,N}+\lambda I)^{-1}g_N}_{L^2(\X^2)}}_{[A]}\nonumber \\
&+\underbrace{\norm{T_k^{\frac12}((T_{k,N}+\lambda I)^{-1}T_{k,N} f_0-f_0)}_{L^2(\X^2)}}_{[B]} \label{eq:norm_bound_upper}
\end{align}

 \textbf{Step 5:} To get bounds on [A] and [B], we rely on Hoeffding-like concentration inequalities in Hilbert spaces.
 \begin{lemma}[\cite{pinelis1994optimum}]\label{lem:concentration}
Let $\mathcal{H}$ be a Hilbert space and $\xi$ be a random variable with values in $\mathcal{H}$. Assume that $\|\xi\|_{\mathcal{H}} \leq M$ almost surely. Let $\left\{\xi_1, \xi_2, \ldots, \xi_N\right\}$ be a sample of $N$ independent observations for $\xi$. Then for any $0<\delta<1$,
$$
\left\|\frac{1}{N} \sum_{i=1}^N\left[\xi_i-\text{E}(\xi)\right]\right\|_{\mathcal{H}} \leq \frac{2 M \log (2 / \delta)}{N}+\sqrt{\frac{2 \text{E}\left(\|\xi\|_{\mathcal{H}}^2\right) \log (2 / \delta)}{N}}
$$
with confidence at least $1-\delta$.
\end{lemma}
 Lemma~\ref{lem:concentration} can now be applied to the random variable $\xi_A=(T_k+\lambda I)^{-\frac12} \innerpro{G_k^{\frac12}m_{\x},.}_{L^2(\X)}G_k^{\frac12}m_{\x}$, which takes values in the space of Hilbert Schmidt operators in $L^2(\X)$.
 We denote the associated norm by $\norm{.}_{\text{HS}(L^2(\X))}$.
 The expectation $\text{E}$ is to be taken with respect to the sampling process of the functional input data, i.e. the empirical mean embeddings in our case.
 Combining Lemma \ref{lemma:uniform bound_mean_embeddings} (which gives a uniform bound on the input functional data) with the spectral decomposition of $T_k$, one obtains $\norm{\xi_A}_{\text{HS}(L^2(\X))}\le \frac{\kappa^{10}}{\sqrt{\lambda}}$ and $\text{E}(\norm{\xi_A}^2_{\text{HS}(L^2(\X))}) \le \kappa^{10} \gamma_{T_k}(\lambda)$. Applying Lemma~\ref{lem:concentration} yields
 \begin{align}
     \norm{(T_k+\lambda I)^{-\frac12}(T_k-T_{k,N})}_{\text{Op}(L^2(\X))} &\le \frac{2 \kappa^5 }{\sqrt{N}} \left(\frac{\kappa^5}{\sqrt{\lambda N} }+\sqrt{\gamma_{T_k}(\lambda)}\right) \log (2 / \delta)\nonumber\\
     &=:C_1(N,\gamma) \log (2 / \delta), \label{eq:norm_diff_t_temp}
 \end{align}
see~\cite[Proposition~3.2]{tong2022non} for details.

\textbf{Step 6:} Using the relation 
\begin{align*}
    BA^{-1}=(B-A)B^{-1}(B-A)A^{-1}+(B-A)B^{-1}+I
\end{align*}
for the product of invertible operators $A$ and $B$ on Banach spaces, together with Eq.~\eqref{eq:norm_diff_t_temp} and the bounds $\norm{(T_{k,N}+\lambda I)^{-\frac12}}_{\text{Op}(L^2(\X))} \le \frac{1}{\lambda}$, $\norm{(T_k+\lambda I)^{-\frac12}}_{\text{Op}(L^2(\X))} \le \frac{1}{\sqrt{\lambda}}$ (which follow from the spectral theorem), we get
\begin{align} \label{eq:mat_sin_main_3}
    &\norm{(T_k+\lambda I)(T_{k,N}+\lambda I)^{-1}}_{\text{Op}(L^2(\X))}  \nonumber \\
    &\le \frac{1}{\lambda}\norm{(T_k+\lambda I)^{-\frac12}(T_k-T_{k,N})}^2_{\text{Op}(L^2(\X))}+\frac{1}{\sqrt{\lambda}}\norm{(T_k+\lambda I)^{-\frac12}(T_k-T_{k,N})}_{\text{Op}(L^2(\X))} +1 \nonumber \\ &\le \left( \frac{C_1(N,\lambda) \log (4 / \delta)}{\sqrt{\lambda}}+1\right)^2
\end{align}
with confidence at least $1-\frac{\delta}{2}$, see~\cite[Proposition 3.3]{tong2022non} for details.

\textbf{Step 7:} The inequality
\begin{align*}
    \norm{A^{\alpha}B^{\alpha}}_{\text{Op}(L^2(\X))} \le \norm{AB}^{\alpha}_{\text{Op}(L^2(\X))},
\end{align*}
valid for positive operators $A$, $B$ on Hilbert spaces and $0 < \alpha <1$, allows us to bound [A] in Eq.~\eqref{eq:mat_sin_main_2} as follows:
\begin{align*}
    \norm{T_k^{\frac12}(T_{k,N}+\lambda I)^{-1}g_N}_{L^2(\X^2)} \le \norm{(T_k+\lambda I)(T_{k,N}+\lambda I)^{-1}}_{\text{Op}(L^2(\X))} \norm{(T_k+\lambda I)^{-\frac12}g_N}_{L^2(\X^2)}.
\end{align*}
To deal with $\norm{(T_k+\lambda I)^{-\frac12}g_N}_{L^2(\X^2)}$, we apply Chebyshev's inequality to the random variable 
\begin{align*}
    \xi_B=(T_k+\lambda I)^{-\frac12} \bar{\varepsilon}(s) (G_k^{\frac12} m_{\x})(t) 
\end{align*}
with values in $L^2(\X^2)$. Using the spectral decomposition of $T_k$, and the assumptions (B3) and (B4) from \ref{subsec:assumptions} on $\bar{\varepsilon}$, we obtain $\text{E}(\xi_B)=0$ and $\text{E}(\norm{\xi_B}_{L^2(\X^2)}^2) =\sigma^2 \gamma_{T_k}(\lambda)$ and thus:
\begin{align*}
    \norm{(T_k+\lambda I)^{-\frac12}g_N}_{L^2(\X^2)} \le \frac{2 \bar{\sigma}}{\delta}\sqrt{\frac{\gamma_{T_k}(\lambda)}{n}} \le \frac{\bar{\sigma} C_1(N,\lambda)}{\kappa^5 \delta},
\end{align*}
with probability at least $1-\frac{\delta}{2}$.
Using Eq.~\eqref{eq:mat_sin_main_3}, we get
\begin{align} \label{eq:mat_sin_main_4}
    [A] \le \frac{\bar{\sigma}(\log (4 / \delta))^2}{\kappa^5 \delta} \left( \frac{C_1(N,\lambda) }{\sqrt{\lambda}}+1\right)^2 C_1(N,\lambda),
\end{align}
with probability $1-\delta$. See \cite[Theorem 3.4]{tong2022non} for details.

\textbf{Step 8:} Using $(T_{k,N}+\lambda I)^{-1}T_{k,N} f_0-f_0=-\lambda(T_{k,N}+\lambda I)^{-1}f_0$ we obtain
\begin{align} \label{eq:mat_sin_main_5}
    [B] \le \sqrt{\lambda} \left( \frac{C_1(N,\lambda) \log (4 / \delta)}{\sqrt{\lambda}}+1\right) \norm{f_0}_{L^2(\X^2)},
\end{align}
with probability $1-\frac{\delta}{2}$ on the same event as Eq.~\eqref{eq:norm_diff_t_temp}. See~\cite[Theorem 3.5]{tong2022non} for details.

\textbf{Step 9:}
To finish the proof of Corollary \ref{cor:func_reg}, we set $\lambda=N^{-\frac{1}{1+c_6}}$ and observe that
\begin{align} \label{eq:C_1_est}
    C_1(N,\lambda) \le 2\kappa^5(\kappa^5+\sqrt{c_5})\sqrt{\lambda}.
\end{align}
Putting everything together, we end up with
\begin{align*}
    \int_{\M(\X)}  &\int_{\X^n} \norm{(G_{\z^{(1)},\ldots,\z^{(N)}}^{\lambda_1,\ldots,\lambda_N,\lambda}-G) m_{\x}}_{L^2(\X)}^2 \diff P_\X^n(\x) \diff E(P) \le 2([A]+[B]) \\
    &\le \frac{\log \frac{4}{\delta}}{\delta^2} N^{-\frac{1}{1+c_6}} 2\left(\frac{\Bar{\sigma}^2\left(2 \kappa^5   \left(\kappa^5  +\sqrt{c_5}\right)+1\right)^6}{\kappa^{10} }+\left(2 \kappa^5  \left(\kappa^5 +\sqrt{c_5}\right)+1\right)^2 \norm{f_0}_{L^2(\X^2)}^2\right),
\end{align*}
where the first inequality follows from Eq.~\eqref{eq:mat_sin_main_1}, Eq.~\eqref{eq:mat_sin_main_2} and Eq.~\eqref{eq:mat_sin_main_3}, and the last one from Eq.~\eqref{eq:mat_sin_main_4}, Eq.~\eqref{eq:mat_sin_main_5} in combination with Eq.~\eqref{eq:C_1_est}, assumption (A5) (to control the effective dimension) and the specific choice of $\lambda$.
\hfill
\end{proof}

\begin{proof}[Proof of Lemma \ref{lemma:operator_norm}]
\label{proof:operator_norm}
To upper bound $\norm{G_{\z^{(1)},\ldots,\z^{(N)}}^{\lambda_1,\ldots,\lambda_N,\lambda}-G}^2_{\text{Op}(L^2(\X))}$, we first observe that it is enough to upper bound $\norm{\beta_{\z^{(1)},\ldots,\z^{(N)}}^{\lambda_1,\ldots,\lambda_N,\lambda}-\beta}^2_{L^2(\X^2)}$ (i.e. the Hilbert Schmidt norm of the associated integral operators). Thus, recalling \eqref{eq:def_f_fN}, we equivalently aim to control
$\norm{G_k^{\frac12}(f_0-f_{N,\lambda})}_{L^2(\X^2)}$. Using the coercivity assumption (B2), we deduce that
\begin{align*}
    \norm{G_k^{\frac12} f}^2_{L^2(\X^2)}&=\innerpro{G_k^{\frac12} f,G_k^{\frac12} f}_{L^2(\X^2)}\le c_7 \innerpro{G_C G_k^{\frac12} f,G_k^{\frac12} f}_{L^2(\X^2)}\\
    &=c_7 \innerpro{G_k^{\frac12} G_C G_k^{\frac12} f, f}_{L^2(\X^2)}=c_7\innerpro{T_k f, f}_{L^2(\X^2)}=c_7\norm{T_k^{\frac12}f}^2_{L^2(\X^2)}
\end{align*}
for all $f \in L^2(\X^2)$. Setting $f=f_0-f_{N,\lambda}$, we obtain:
\begin{align*}
    \norm{G_k^{\frac12}(f_0-f_{N,\lambda})}^2_{L^2(\X^2)} \le c_7 \norm{T_k^{\frac12}(f_0-f_{N,\lambda})}^2_{L^2(\X^2)}, 
\end{align*}
so that the lemma is proven by Eq.~\eqref{eq:mat_sin_main_1} applying the result from Corollary \ref{cor:func_reg}.
\end{proof}

\section*{Acknowledgements}

The research reported in this paper has been partly funded by the Federal Ministry for Climate Action, Environment, Energy, Mobility, Innovation and Technology (BMK), the Federal Ministry for Digital and Economic Affairs (BMDW), and the Province of Upper Austria in the frame of the COMET–Competence Centers for Excellent Technologies Programme and the COMET Module S3AI managed by the Austrian Research Promotion Agency FFG.

\bibliography{domgen.bib}

\begin{thebibliography}{20}
\providecommand{\natexlab}[1]{#1}
\providecommand{\url}[1]{\texttt{#1}}
\expandafter\ifx\csname urlstyle\endcsname\relax
  \providecommand{\doi}[1]{doi: #1}\else
  \providecommand{\doi}{doi: \begingroup \urlstyle{rm}\Url}\fi

\bibitem[Baxter(1998)]{baxter1998theoretical}
J.~Baxter.
\newblock Theoretical models of learning to learn.
\newblock In \emph{Learning to learn}, pages 71--94. Springer, 1998.

\bibitem[Blanchard et~al.(2011)Blanchard, Lee, and
  Scott]{blanchard2011generalizing}
G.~Blanchard, G.~Lee, and C.~Scott.
\newblock Generalizing from several related classification tasks to a new
  unlabeled sample.
\newblock \emph{Advances in neural information processing systems}, 24, 2011.

\bibitem[Blanchard et~al.(2021)Blanchard, Deshmukh, Dogan, Lee, and
  Scott]{blanchard2021domain}
G.~Blanchard, A.~A. Deshmukh, {\"U}.~Dogan, G.~Lee, and C.~Scott.
\newblock Domain generalization by marginal transfer learning.
\newblock \emph{The Journal of Machine Learning Research}, 22\penalty0
  (1):\penalty0 46--100, 2021.

\bibitem[Caponnetto and {De Vito}(2007)]{caponnetto2007optimal}
A.~Caponnetto and E.~{De Vito}.
\newblock Optimal rates for the regularized least-squares algorithm.
\newblock \emph{Foundations of Computational Mathematics}, 7\penalty0
  (3):\penalty0 331--368, 2007.

\bibitem[Cucker and Smale(2002)]{cucker2002mathematical}
F.~Cucker and S.~Smale.
\newblock On the mathematical foundations of learning.
\newblock \emph{Bulletin of the American mathematical society}, 39\penalty0
  (1):\penalty0 1--49, 2002.

\bibitem[Dudley(2002)]{dudley2018real}
R.~M. Dudley.
\newblock \emph{Cambridge Studies in Advanced mathematics: {R}eal Analysis and
  Probability}.
\newblock 74. Cambridge University Press, 2nd edition, 2002.
\newblock \doi{10.1017/CBO9780511755347}.

\bibitem[Engl et~al.(1996)Engl, Hanke, and Neubauer]{engl1996regularization}
H.~W. Engl, M.~Hanke, and A.~Neubauer.
\newblock \emph{Regularization of inverse problems}, volume 375.
\newblock Springer Science \& Business Media, 1996.

\bibitem[Evgeniou et~al.(2005)Evgeniou, Micchelli, Pontil, and
  Shawe-Taylor]{evgeniou2005learning}
T.~Evgeniou, C.~A. Micchelli, M.~Pontil, and J.~Shawe-Taylor.
\newblock Learning multiple tasks with kernel methods.
\newblock \emph{Journal of machine learning research}, 6\penalty0 (4), 2005.

\bibitem[Gretton et~al.(2006)Gretton, Borgwardt, Rasch, Sch{\"o}lkopf, and
  Smola]{gretton2006kernel}
A.~Gretton, K.~Borgwardt, M.~Rasch, B.~Sch{\"o}lkopf, and A.~Smola.
\newblock A kernel method for the two-sample-problem.
\newblock \emph{Advances in neural information processing systems}, 19, 2006.

\bibitem[Guo et~al.(2017)Guo, Lin, and Zhou]{guo2017learning}
Z.-C. Guo, S.-B. Lin, and D.-X. Zhou.
\newblock Learning theory of distributed spectral algorithms.
\newblock \emph{Inverse Problems}, 33\penalty0 (7):\penalty0 074009, 2017.

\bibitem[Jin et~al.(2022)Jin, Lu, Blanchet, and Ying]{jin2022minimax}
J.~Jin, Y.~Lu, J.~Blanchet, and L.~Ying.
\newblock Minimax optimal kernel operator learning via multilevel training.
\newblock \emph{arXiv preprint arXiv:2209.14430}, 2022.

\bibitem[Maurer(2005)]{maurer2005algorithmic}
A.~Maurer.
\newblock Algorithmic stability and meta-learning.
\newblock \emph{Journal of Machine Learning Research}, 6:\penalty0 967--994,
  2005.

\bibitem[Mollenhauer et~al.(2022)Mollenhauer, M{\"u}cke, and
  Sullivan]{mollenhauer2022learning}
M.~Mollenhauer, N.~M{\"u}cke, and T.~Sullivan.
\newblock Learning linear operators: Infinite-dimensional regression as a
  well-behaved non-compact inverse problem.
\newblock \emph{arXiv preprint arXiv:2211.08875}, 2022.

\bibitem[Muandet et~al.(2013)Muandet, Balduzzi, and
  Sch{\"o}lkopf]{muandet2013domain}
K.~Muandet, D.~Balduzzi, and B.~Sch{\"o}lkopf.
\newblock Domain generalization via invariant feature representation.
\newblock In \emph{International Conference on Machine Learning}, pages 10--18.
  PMLR, 2013.

\bibitem[Pinelis(1994)]{pinelis1994optimum}
I.~Pinelis.
\newblock Optimum bounds for the distributions of martingales in banach spaces.
\newblock \emph{The Annals of Probability}, pages 1679--1706, 1994.

\bibitem[Sch{\"o}lkopf et~al.(2001)Sch{\"o}lkopf, Herbrich, and
  Smola]{scholkopf2001generalized}
B.~Sch{\"o}lkopf, R.~Herbrich, and A.~J. Smola.
\newblock A generalized representer theorem.
\newblock In \emph{Annual Conference on Computational Learning Theory (COLT)},
  pages 416--426. Springer, 2001.

\bibitem[Sriperumbudur et~al.(2010)Sriperumbudur, Gretton, Fukumizu,
  Sch{\"o}lkopf, and Lanckriet]{sriperumbudur2010hilbert}
B.~K. Sriperumbudur, A.~Gretton, K.~Fukumizu, B.~Sch{\"o}lkopf, and G.~R.
  Lanckriet.
\newblock Hilbert space embeddings and metrics on probability measures.
\newblock \emph{The Journal of Machine Learning Research}, 11:\penalty0
  1517--1561, 2010.

\bibitem[Szab{\'o} et~al.(2016)Szab{\'o}, Sriperumbudur, P{\'o}czos, and
  Gretton]{szabo2016learning}
Z.~Szab{\'o}, B.~K. Sriperumbudur, B.~P{\'o}czos, and A.~Gretton.
\newblock Learning theory for distribution regression.
\newblock \emph{The Journal of Machine Learning Research}, 17\penalty0
  (1):\penalty0 5272--5311, 2016.

\bibitem[Tong et~al.(2022)Tong, Hu, and Ng]{tong2022non}
H.~Z. Tong, L.~F. Hu, and M.~Ng.
\newblock Non-asymptotic error bound for optimal prediction of
  function-on-function regression by rkhs approach.
\newblock \emph{Acta Mathematica Sinica, English Series}, 38\penalty0
  (4):\penalty0 777--796, 2022.

\bibitem[Wolfer and Alquier(2022)]{wolfer2022variance}
G.~Wolfer and P.~Alquier.
\newblock Variance-aware estimation of kernel mean embedding.
\newblock \emph{arXiv preprint arXiv:2210.06672}, 2022.

\end{thebibliography}

\end{document}